\documentclass[conference]{IEEEtran}

\usepackage{newtxtext}
\usepackage{amsmath,amsfonts,amsthm}

\usepackage{amssymb}
\usepackage{newtxmath}

\usepackage{cite}
\usepackage{amsmath,amsfonts}
\usepackage{graphicx}
\usepackage{textcomp}
\usepackage{xcolor}
\usepackage[T1]{fontenc}
\usepackage[utf8]{inputenc}
\usepackage{bm}
\usepackage{comment}
\usepackage{mathtools}
\usepackage{booktabs}
\usepackage{multirow}
\usepackage{array}
\usepackage{enumitem}
\PassOptionsToPackage{hyphens}{url}
\usepackage{xurl}
\usepackage[hidelinks]{hyperref}

\hypersetup{pdfauthor={},pdftitle={}}

\newtheorem{definition}{Definition}
\newtheorem{theorem}{Theorem}
\newtheorem{proposition}{Proposition}

\newtheorem{remark}{Remark}

\newcommand{\TODO}[1]{\textcolor{red}{#1}}
\newcommand{\cpsaint}{\textsc{CPSAINT}}
\newcommand{\friesak}{\textsc{FRIESA-K}}
\newcommand{\risk}{\mathcal{R}}
\newcommand{\states}{\mathcal{S}}
\newcommand{\paths}{\Pi}
\newcommand{\controls}{\mathcal{C}}

\newcommand{\layers}{\mathcal{L}}

\newcommand{\Fcat}{\mathcal{F}}
\newcommand{\Wwarn}{\mathcal{W}}
\newcommand{\Rgov}{R_{\mathrm{gov}}}
\newcommand{\Ri}{R_i}
\newcommand{\Rres}{R_{\mathrm{res}}}
\def\BibTeX{{\rm B\kern-.05em{\sc i\kern-.025em b}\kern-.08em
    T\kern-.1667em\lower.7ex\hbox{E}\kern-.125emX}}
\begin{document}

\title{From Agent Failure Paths to Quantified Residual Risk:
A Compositional Framework for Resilient Agentic AI}

\author{
\IEEEauthorblockN{
  Hassan Karim\IEEEauthorrefmark{1},
  Sai Sitharaman\IEEEauthorrefmark{2},
  Deepti Gupta\IEEEauthorrefmark{3},
  Danda B. Rawat\IEEEauthorrefmark{4}
}
\IEEEauthorblockA{
  \IEEEauthorrefmark{1}Stable Cyber, Dallas, TX, USA; ORCID: 0000-0002-5441-049X \\
  \IEEEauthorrefmark{2}Zetafence, Dublin, CA, USA \\
  \IEEEauthorrefmark{3}Texas A\&M University--Central Texas, Killeen, TX, USA \\
  \IEEEauthorrefmark{4}Howard University, Washington, DC, USA; ORCID: 0000-0003-3638-3464
}
}

\maketitle
\begin{abstract}
Agentic AI is crossing trust boundaries faster than current risk models can represent. 
Existing approaches provide one of two partial views. They either describe failure mechanisms without producing a transferable residual-risk estimate, or they produce a risk estimate while treating the internal failure path as a black box. 
We couple those two views by proposing \cpsaint, a seven-layer integrity decomposition over Physical state, Sensors, Data, Compute, Actuators, Environment, and Time, paired with \friesak, a residual-risk functional that maps each failure path to a quantified risk instance.
\friesak\ grounds the resistance term $K$ in a controlled absorbing Markov model so that control effectiveness is derived from state dynamics rather than assigned as an informal score.
The result is a concise mechanism-to-magnitude pipeline for resilient agentic and embodied AI. 
We report governance observability through a separate additive penalty instead of inserting governance as a new variable in the resistance functional. We formalize structural composability linking valid failure paths to well-defined risk instances and show the framework on two contrasting scenarios: a hard real-time warehouse robot and a governance-instrumented financial-services agent. Across both cases, the same layer grammar, variable semantics, and dynamic-resistance construction remain intact. Thus, we obtain a compact kernel that supports cross-domain reasoning, explicit assumptions, and quantitatively grounded formalism of composable trust.
\end{abstract}


\section{Introduction}
Agentic AI, software that autonomously pursues multi-step goals through planning, tool use, and environmental interaction, is moving rapidly from prototype to production. 
Gartner projects that by 2028, at least 15\% of
  day-to-day work decisions will be made autonomously by such systems, up from
  essentially zero in 2024, and that 33\% of enterprise applications will embed
  agentic capabilities over the same period~\cite{Gartner2025AgenticCancel}.
  
Unlike chat-bots that merely respond to a single prompt, an agentic system decomposes objectives, sequences actions across multiple decision cycles,
and modifies external state (records, workflows, physical actuators) with little or no human approval per step.

Agentic AI is already operating across trust boundaries that static risk scores do not capture~\cite{dehghantanha2026sokatksurfagentic}.
A planning agent, for example, may call tools, update records, trigger workflows, or influence actuation. A robotic controller may fuse sensor data, invoke learned models, and issue time-critical commands.
In both settings, the failure is path-dependent. Integrity does not fail once. It fails across layers, interfaces, and deadlines.
\begin{figure}[t]
\centering
\includegraphics[width=0.99\columnwidth]{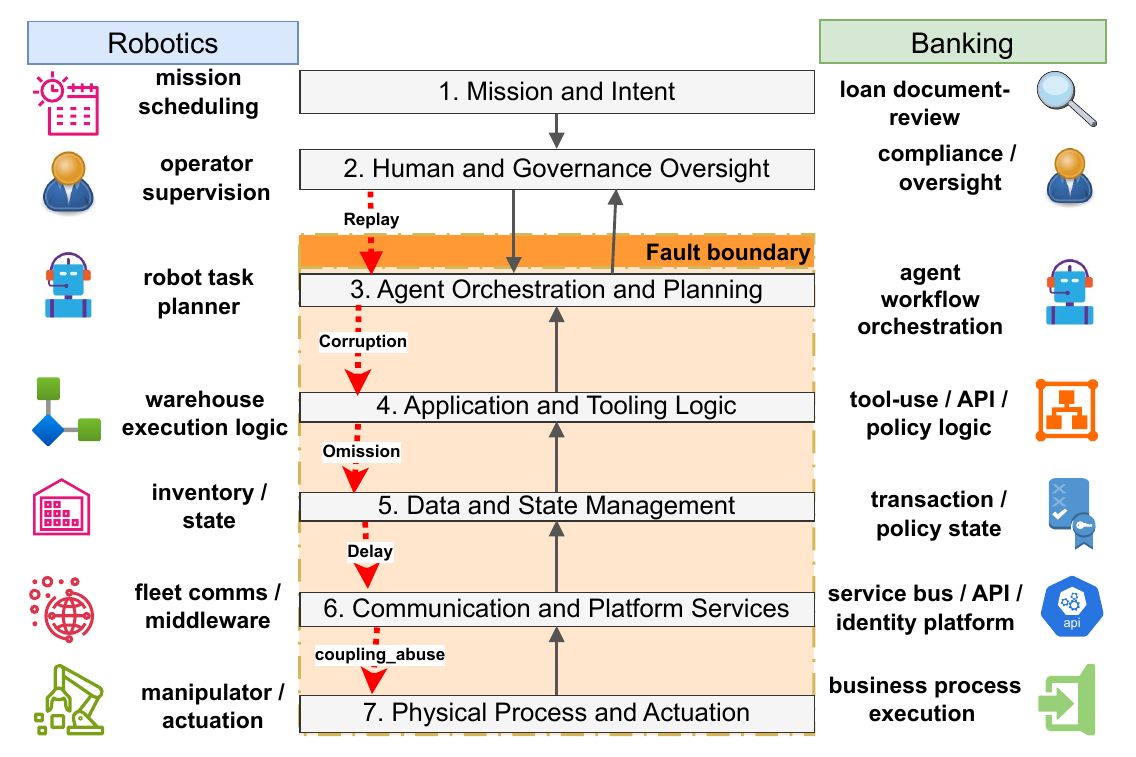}%
\caption{Comparative \cpsaint\ seven-layer grammar for structural decomposition across warehouse robotics and banking agentic workflows. Initiating compromise begins near orchestration, data, or tooling logic, crosses explicit fault boundaries, and propagates toward either physical actuation or governance-sensitive workflow degradation. Layer numbering corresponds to the formal symbol set 
given in Section~\ref{sec:method}.}
\label{fig:cpsaint}
\end{figure}

Current methods split the problem. Structural methods such as STAMP~\cite{Leveson2004STAMP} and STPA~\cite{Friedberg2017STPASafeSec} and related safety-security analyses provide causal pathways, unsafe control scenarios, and constraint violations, but they stop before a transferable residual-risk estimate. Quantitative approaches such as adversarial risk analysis, Conditional Value at Risk (CVaR) reporting~\cite{RockafellarUryasev2000CVaR}, and rare-event simulation provide estimators and tail summaries, but they usually treat the system between threat and loss as a black box~\cite{RiosInsua2021ARAcyber}.
Recent agent-security and runtime-assurance methods strengthen privilege boundaries, constrain tool use, or verify local behaviors, yet they do not provide a compact, domain-transferable mapping from failure path to residual risk.  

No existing framework simultaneously answers both questions:
  \textit{through what mechanism does integrity fail?} and
  \textit{how much residual risk results?}
The closest work either provides structural causality without a transferable risk estimate, or a risk estimate without structural grounding. These ensure that neither is suitable for the cross-domain, path-dependent failure modes that agentic and cyber-physical systems exhibit.
  
In this paper, we address this gap by coupling mechanism and magnitude. We use \cpsaint\ (Fig.~\ref{fig:cpsaint}) to answer a structural question: through what mechanism does integrity fail?
We use \friesak\ to answer the question of how much residual risk results, quantitatively? We then derive the resistance term $K(\pi,\tau,u)$ from controlled absorbing dynamics so that control effectiveness reflects measurable change in catastrophe probability by horizon, not an arbitrary expert score.
We use the \friesak\ semantics, herein, and report governance-observability degradation through a separate additive term rather than redefining the base equation.

The main contributions of this paper are as follows:
\begin{itemize}
\item First, we define \cpsaint\ as a seven-layer structural decomposition with explicit path semantics.
\item Second, we define \friesak\  with locked variable boundaries and dynamic construction for $K(\pi,\tau,u)$.
\item Third, we state a composition theorem showing that a valid \cpsaint\ path induces a well-defined \friesak\ risk instance without changing the functional form.
\item Fourth, we demonstrate the framework on two materially different use cases: a real-time warehouse robot and a banking/financial-services agent.
\end{itemize}
The remainder of this paper is organized as follows. Section~\ref{sec:related-works} compresses the relevant literature.
Section~\ref{sec:system} defines the system model, notation, and threat boundary.
Section~\ref{sec:method} presents the proposed method: \cpsaint, \friesak, the dynamic-resistance construction, and the composition theorem.
Section~\ref{sec:evaluation} evaluates the framework analytically on two scenarios and reports the resulting sensitivity behavior.
Sections~\ref{sec:limits} and~\ref{sec:conclusion} provide limitations and conclusions.

\section{Related Works}
\label{sec:related-works}
This work builds on three bodies of work: structural decomposition of failure causes, quantitative adversarial risk estimation, runtime assurance, security, and governance for AI and embodied AI agents.

\label{sec:rw:decomposition}
The closest prior work falls into three families. Structural safety-security methods such as STAMP/STPA~\cite{Leveson2004STAMP}, STPA-SafeSec~\cite{Friedberg2017STPASafeSec}, and attack-defense trees~\cite{Kordy2014ADTrees} provide causal pathways, attacker-defender logic, and unsafe-control reasoning, but they do not map decomposed failure paths into a stable residual-risk model.

Modular agent frameworks such as DSPy~\cite{Khattab2024DSPy}, AutoGen~\cite{Wu2024AutoGen}, and MetaGPT~\cite{Hong2024MetaGPT} show that composition matters operationally, but they do not explicitly deal with fault boundaries, propagation semantics, and residual-risk accounting.
Empirical taxonomies such as MAST catalog failure patterns in multi-agent traces, but they do not supply a mechanism-to-magnitude calculus~\cite{Cemri2025MAST}.

\label{sec:rw:quantrisk}
Quantitative risk and tail-estimation methods provide the other half of the puzzle.
Adversarial Risk Analysis~\cite{RiosInsua2021ARAcyber} contributes a rigorous decision-theoretic treatment of intentional threats.
Coherent-risk theory, CVaR, ~\cite{Artzner1999CoherentRisk,RockafellarUryasev2000CVaR}, subset simulation~\cite{AuBeck2001SubsetSimulation}, and related rare-event methods provide the mathematical basis for conservative tail reporting and small-probability estimation.
Cascading-failure models show how small perturbations can propagate into systemic loss~\cite{Buldyrev2010Cascades}.
However, these lines of work mostly begin after the internal structural pathway has already been abstracted away making it difficult to leverage those models for the granular risk modeling needed for real-time agentic systems.

\label{sec:rw:runtime}

Runtime-assurance and security-containment work narrows behavior but typically at the level of a single agent or a single tool boundary. AgentSpec~\cite{Wang2026AgentSpec}, Pro2Guard~\cite{Wang2025Pro2Guard}, CaMeL~\cite{Debenedetti2025CaMeL}, IsolateGPT~\cite{Wu2025IsolateGPT}, and Progent~\cite{Shi2025Progent} all illustrate some type of constraint enforcement, probabilistic model checking, capability isolation, or privilege control. However, they do not propagate residual risk, thus losing risk fidelity across a composed agentic path with fixed semantics.
Dependable-AI work makes the final point. Assurance 2.0~\cite{BloomfieldRushby2024Assurance}, guaranteed-safe AI, and recent topology-resilience studies argue that architecture and explicit assumptions determine resilience~\cite{Huang2025Resilience}, but they do not give a cross-domain residual-risk functional grounded in structural failure paths.
Prior work~\cite{Karim2025IoRT} operationalizes NIST AI RMF governance functions at the workload level for LLM-integrated robotic systems, and in doing so, provides the empirical basis for \friesak's treatment of governance observability as a modeled state dimension rather than an external policy overlay.
We target that intersection. To our knowledge, no existing approach jointly fixes the failure grammar, the risk semantics, and the dynamic control interpretation in a single framework.

\section{System and Threat Model}
\label{sec:system}

\begin{table}[t]
\caption{Core notation}
\label{tab:notation}
\centering
\begin{tabular}{ll}
\toprule
Symbol & Meaning \\
\midrule
$\layers$ & Fixed layer set $\{P,S,D,C,A,E,T\}$ \\
$\Phi$ & Valid cross-layer propagation relation \\
$\paths$ & Set of valid attack propagation or failure paths \\
$\pi$ & Specific path in $\paths$ \\
$\tau$ & Finite evaluation horizon (s); domain-dependent \\
$\states$ & Finite controlled state space \\
$\controls$ & Admissible defender policies \\
$K(\pi,\tau,u)$ & Dynamic control effectiveness for path $\pi$ by horizon \\
$w_{\mathrm{gov}}$ & Governance weight scalar ($\geq 0$) \\
$\Rgov(\pi,\tau,u)$ & Governance dwell-time observability penalty \\
$\Wwarn$ & Warning states: degraded but non-catastrophic \\
$\pi_F(\pi,\tau,u)$ & Catastrophe-absorption probability by $\tau$ \\
$\risk(\pi,\tau,u)$ & \friesak\ risk for failure path $\pi$ \\
$\Rres$ & Reported residual-risk score with governance term \\
\bottomrule
\end{tabular}
\end{table}

We model an agentic or cyber-physical system as a labeled directed graph
\begin{equation}
\label{eq:gcps}
G_{\mathrm{CPS}}=(\layers,\Phi,\mathcal{M},\paths),
\end{equation}
where $\layers=\{P,S,D,C,A,E,T\}$ is the fixed \cpsaint\ layer set, $\Phi\subseteq \layers\times\layers$ is the valid propagation relation, $\mathcal{M}=\{corruption, delay, omission, replay, coupling\_abuse\}$ is the failure-mode alphabet, and $\paths$ is the set of valid failure paths admitted by the instantiated architecture. The graph is domain-agnostic.
Domain transfer occurs through instantiation of nodes, edges, and path content, not by changing the layer grammar itself.

We assume an adversary that seeks to induce integrity loss along one or more valid paths. Entry may occur through sensing, data ingestion, compute logic, tool execution, timing manipulation, or environmental coupling. Defender actions are encoded by a control policy $u\in\controls$. The evaluation horizon $\tau$ is finite and domain-dependent. For hard real-time systems it may be milliseconds or seconds. For enterprise agents it may be minutes or longer. A single $\tau$ may also represent a narrow operational analysis window within a domain whose full governance or audit lifecycle spans much longer intervals; the banking-agent instantiation in Section~\ref{sec:caseB} uses $\tau=0.75$\,s as a short-window operational slice, not as a model of the full assurance horizon.

Table~\ref{tab:notation} summarizes the core notation used throughout the paper.

\section{Proposed Method}
\label{sec:method}

\subsection{\cpsaint: Structural Failure Decomposition}

We define \cpsaint\ over the ordered layer tuple in Eq. (\ref{eq:failuretuple}):
\begin{equation}
\label{eq:failuretuple}
\Sigma=\langle P,S,D,C,A,E,T\rangle,
\end{equation}
where $P$ denotes physical plant or mission-relevant hardware state, $S$ sensing and signal acquisition, $D$ data/information representation and model inputs, $C$ compute and decision logic, $A$ actuation or side-effectful execution, $E$ external environment or counterparty context, and $T$ timing and synchronization. Each layer may exhibit one of five integrity-failure modes: corruption, delay, omission, replay, or coupling abuse.

\begin{definition}[Agent failure path]
A valid agent failure path is an ordered sequence (Eq. \ref{eq:agentfailurepath})
\begin{equation}
\label{eq:agentfailurepath}
\pi=[(\ell_1,m_1),(\ell_2,m_2),\ldots,(\ell_n,m_n)]
\end{equation}
such that each $\ell_i\in\layers$, each $m_i\in\mathcal{M}$, and each successive pair respects the propagation relation $\Phi$ of the instantiated system.
\end{definition}

The purpose of this construction, illustrated in Figure~\ref{fig:cpsaint}, is to establish a stable grammar for integrity propagation. In this way, a warehouse robot, a medical workflow agent, and a banking assistant can be expressed with the same layer set while preserving domain-specific path content.

Figure~\ref{fig:cpsaint} renders the layer set as an operational stack to show how failure propagates across architectural boundaries in each domain instantiation.
The formal layer symbols map to the figure as follows.
Layers~1--2 (Mission and Intent; Human and Governance Oversight) instantiate $E$, the external environment and counterparty context, which encompasses operator intent, governance principals, and regulatory constraints.
Layer~3 (Agent Orchestration and Planning) and Layer~4 (Application and Tooling Logic) both instantiate $C$, compute and decision logic, at different abstraction levels within the same domain.
Layer~5 (Data and State Management) instantiates $D$.
Layer~6 (Communication and Platform Services) instantiates $S$, sensing and signal acquisition, which in software-intensive agents includes API endpoints, message buses, and data ingestion interfaces.
Layer~7 (Physical Process and Actuation) jointly instantiates $P$ (physical plant) and $A$ (actuation or side-effectful execution).
$T$ (timing and synchronization) is a cross-cutting property: it governs the temporal constraints at each layer rather than occupying a dedicated architectural slot, consistent with the design principle that network delay maps to $T$ and data routing integrity maps to $D$.
\subsection{\friesak: Residual-Risk}

\friesak\ is defined in Equation (\ref{eq:friesak}):
\begin{equation}
\label{eq:friesak}
\risk(\pi,\tau,u)=\frac{F(\pi)\, R_i(\pi)\,E(\pi)\,S(\pi)\,A(\pi)}{K(\pi,\tau,u)},
\end{equation}

\noindent where, $F$ denotes the frequency of the threat. Reachability is given by $R_i$ and exploitability is expressed in $E$. The severity of the threat is captured with $S$. We captured the potential a threat being amplified with $A$. Many models fail to capture control effectiveness. We use control~$K$ to capture the resistance effectiveness of a control. If a factor makes attacks more likely or consequences more severe, it belongs in the numerator. If a factor makes attacks harder or recovery faster, it belongs in $K$. This separation prevents semantic drift and double counting.

We use a terminal-consequence interpretation for severity. That is, $S(\pi)$ reflects the terminal consequence of the path rather than the sum of intermediate disturbances. Intermediate broadening of impact is captured through $A(\pi)$, not by repeatedly re-counting severity at each step. This rule gives the most stable path semantics.

\subsection{Dynamic Resistance and Governance Observability}

We derive $K$ from path-conditioned controlled stochastic dynamics rather than from an informal scalar assessment as might be the tendency of risk management practitioners. For a fixed valid failure path $\pi$ and chosen analysis horizon $\tau$, let $X_{\pi}(t)$ be a finite-state controlled continuous-time Markov chain (CTMC) on $\states$ with generator $Q_{\pi}(u)$ under policy $u\in\controls$. The path determines the relevant hazard states, interfaces, and control touch points. The horizon $\tau$ remains an external finite evaluation window rather than a function of $\pi$. 
The state space is $\states = \Sigma_{\mathrm{op}} \times \Sigma_{\mathrm{obs}}$, where
$\Sigma_{\mathrm{op}}=\{\sigma_s,\sigma_c,\sigma_d,\sigma_n,\sigma_r,\sigma_f\}$
(safe, compromised, detected, contained, recovered, catastrophic)
and $\Sigma_{\mathrm{obs}}=\{\omega,\bar{\omega}\}$ (observable, unobservable),
giving ten states total.
Let $\Fcat\subset\states$ denote the catastrophic absorbing states and $\Wwarn\subset\states$ the materially degraded but non-catastrophic warning states. 
Transition rates $q_{ij}(u)$ within the generator matrix $Q_{\pi}(u)$ encode empirical propagation, detection, and recovery hazards.
The generator structure and all rate parameters are provided in the open-source
implementation.\footnote{\url{https://github.com/coderhard/friesa-k-DSN-crai2026}}
We define catastrophe probability by horizon $\tau$ as
\begin{equation}
\pi_F(\pi,\tau,u)=\mathbb{P}_u\big(X_{\pi}(\tau)\in\Fcat\mid X_{\pi}(0)=s_{0,\pi}\big).
\end{equation}
We then define dynamic control effectiveness as the baseline-to-controlled benefit factor
\begin{equation}
\label{eq:k_dynamic}
K(\pi,\tau,u)=\frac{\pi_F(\pi,\tau,u_0)}{\pi_F(\pi,\tau,u)}.
\end{equation}
If a control policy lowers catastrophic absorption by horizon for the instantiated path, $K(\pi,\tau,u)$ must be $> 1$. If it does not, the model makes that visible.

This notation resolves the semantics of Eq. \eqref{eq:friesak}: $K$ is path-conditioned because the controlled process is instantiated on the failure path under study, not because the horizon is itself path-valued. Controls deployed on layers or interfaces not traversed by $\pi$ do not alter this instantiated process for that path. When control effects are approximately separable across distinct path transitions, a useful approximation is
\begin{equation}
\label{eq:k_factor_approx}
K(\pi,\tau,u)\approx \prod_{j\,:\,c_j\in\controls(\pi)} k_j(\tau,u),
\end{equation}
where $\controls(\pi)\subseteq\controls$ denotes the path-relevant control set and $k_j(\tau,u)$ is the benefit factor contributed by control $c_j$. We do not assume exact factorization as a general law; overlapping or coupled controls must be modeled jointly to avoid double counting.

To accommodate governance-sensitive settings that require auditability and assurance accounting, we report a separate governance-observability penalty in Eq. \ref{eq:rgov} instead of creating another \friesak\ variable.
\begin{equation}
\label{eq:rgov}
\Rgov(\pi,\tau,u)=\int_0^{\tau}\sum_{s\in\Wwarn}p^{\pi}_{(s,\textsf{NoObs})}(t)\,dt,
\end{equation}
where $\Wwarn\subset\states$ contains materially degraded but non-catastrophic warning states, \textsf{NoObs} marks weak observability, and $p^{\pi}_{(s,\textsf{NoObs})}(t)$ denotes occupancy probability on the same path-instantiated process. We then report the residual score
\begin{equation}
\label{eq:jres}
\Rres(\pi,\tau,u)=\risk(\pi,\tau,u)+w_{\mathrm{gov}}\Rgov(\pi,\tau,u),
\end{equation}
with $w_{\mathrm{gov}}\ge 0$ chosen by governance priority. This preserves a clear separation between operational residual risk and governance-observability degradation.

\begin{figure}[t]
\centering
\parbox{0.75\columnwidth}{%
\centering
\includegraphics[
  width=\linewidth,
  trim=12 28 12 24,
  clip
]{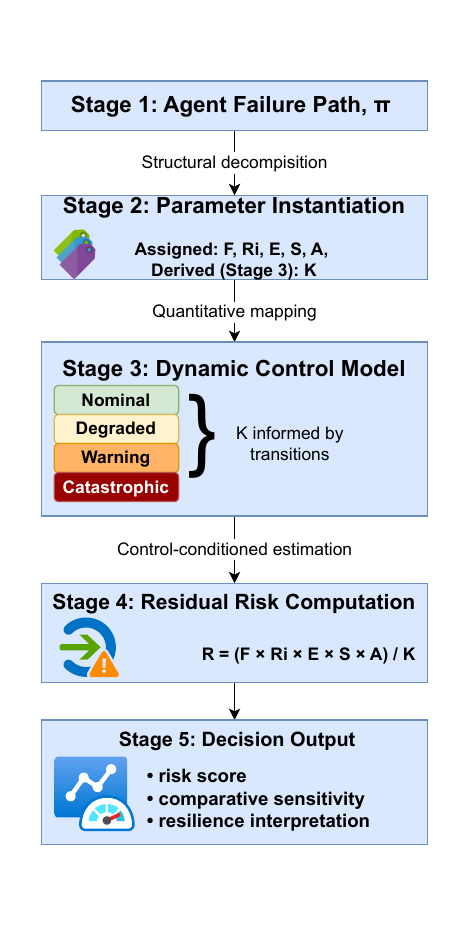}%
}
\caption{Mechanism-to-magnitude pipeline. A valid \cpsaint\ path induces a \friesak\ tuple, while path-conditioned controlled dynamics instantiate $K(\pi,\tau,u)$ and the auxiliary governance-observability penalty.}
\label{fig:pipeline}
\end{figure}

\subsection{Composability}

\begin{definition}[Structural composability]
We call the framework structurally composable if the same fixed layer grammar and failure-mode alphabet can express valid failure paths across domains without redefining the decomposition.
\end{definition}

\begin{definition}[Quantitative composability]
We call the framework quantitatively composable if a valid structural path induces a well-defined \friesak\ instance using defined variable semantics, while preserving the functional form of \eqref{eq:friesak} across domains.
\end{definition}

\begin{theorem}[Composition theorem]
\label{thm:composition}
Let $G_{\mathrm{CPS}}=(\layers,\Phi,\mathcal{M},\paths)$ be an instantiated system and let $\pi\in\paths$ be a valid failure path.
Suppose each step in the attack propagation path $\pi$ admits domain-specific instantiations of $F$, $R_i$, $E$, $S$, and $A$ under the semantics of \friesak, and suppose a defender policy $u$ induces a finite-state absorbing process with well-defined $K(\pi,\tau,u)$ from \eqref{eq:k_dynamic}.
Then $\pi$ induces a well-defined risk instance $\risk(\pi,\tau,u)$ under \eqref{eq:friesak}. 
If governance observability is modeled, the same path also induces a well-defined reported score $\Rres(\pi,\tau,u)$ under \eqref{eq:jres}. 
The mapping preserves functional form across domains. Only path content and parameter values change.
\end{theorem}

\begin{proof}
The proof is constructive. A valid path specifies the ordered layer transitions and failure modes that carry the attack or disturbance forward. We attach $F$ to initiating pressure over the horizon, $\Ri$ to reachable exposure along the path, $E$ to conditional technical feasibility, $S$ to the terminal consequence, and $A$ to cross-layer broadening of impact. By assumption, the path-conditioned controlled CTMC is finite-state with well-defined absorption probability by horizon. Thus, $K(\pi,\tau,u)$ exists and is non-negative. Substituting the instantiated quantities into \eqref{eq:friesak} yields a non-negative risk value. If observability is modeled, \eqref{eq:rgov} is also well-defined on the same state process, which yields the reported score in \eqref{eq:jres}. Because neither the layer grammar nor the variable semantics change, the mapping composes across domains without redefining the risk functional.
\end{proof}

\begin{remark}[Interpretation of control effectiveness]
For fixed $F,\Ri,E,S,A \ge 0$ and $K>0$, the FRIESA-K residual risk
\[
\risk = \frac{F \Ri E S A}{K}
\]
decreases as control effectiveness $K$ increases. Likewise, if
\[
\Rres = \risk + w_{\mathrm{gov}}\Rgov,
\]
with $w_{\mathrm{gov}} \ge 0$, then higher governance degradation yields higher residual governance burden.
\end{remark}

Figure~\ref{fig:pipeline} illustrates the five-stage mechanism-to-magnitude pipeline, from agent failure-path decomposition through dynamic control modeling to residual risk computation and decision output.

\section{Performance Evaluation}
\label{sec:evaluation}

We evaluate the framework through Gillespie SSA fault-injection simulation and deterministic sensitivity analysis across two contrasting scenarios.
The simulation injects failure conditions at the entry layer of each instantiated path and tracks catastrophic absorption probability over the evaluation horizon, producing both a deterministic residual-risk estimate and a Monte Carlo distribution ($N=5{,}000$ trajectories) over uncertain parameter values.
Can one structural decomposition and one risk function support materially different agentic settings without changing
  their semantics?

We answer that question through two scenarios and one shared sensitivity view.

Risk scores throughout reflect the product of instantiated parameter values at bundle-default magnitudes; they are not calibrated to monetary or physical loss units and should be interpreted comparatively, not as absolute loss estimates.

\subsection{Case A: Hard Real-Time Warehouse Robot}
\label{sec:caseA}

We consider an autonomous warehouse robot, $\mathrm{wr}$, that fuses sensor input and model-based planning to avoid collisions under a hard stopping deadline~\cite{karim2026forgebench}. The representative path is given by Eq. \ref{eq:robotfailurepath}:
{\small
\begin{equation}
\label{eq:robotfailurepath}
\begin{split}
\pi_{\mathrm{wr}} ={}& [(S,\textsf{corruption}),(D,\textsf{corruption}),\\
& (C,\textsf{delay}),(T,\textsf{delay}),(A,\textsf{coupling\_abuse})].
\end{split}
\end{equation}
}

The path begins in sensing, propagates through data and compute, violates timing constraints, and culminates in unsafe actuation. Here, the terminal consequence is physical collision or unsafe motion. Thus, $S(\pi_{\mathrm{wr}})$ is high. 
Amplification is also high because delay in $T$ enlarges the consequence of upstream errors. A fast safety interlock, bounded-latency fallback, or PLC-mediated containment logic reduces catastrophic absorption by horizon and therefore increases $K(\pi_{\mathrm{wr}},\tau,u)$ sharply. The governance penalty is secondary. The main question is whether containment outruns escalation. 
Here $\tau=0.75$\,s, $w_{\mathrm{gov}}=0.35$, \texttt{normalize\_gov=true}: $K=1.527$, $\pi_F=0.506$, $\Rres=2{,}563{,}283$. Fault-injection simulation ($N=5{,}000$ SSA trajectories, seed~42): mean $2{,}553{,}038$, p95 $6{,}766{,}279$.

\subsection{Case B: Financial-Services Agent with Governance-Observability Instrumentation}
\label{sec:caseB}

Next, we consider a banking-services agent, that processes customer loan documents. The agent, $\mathrm{fa}$, orchestrates document intake, tool calls, policy checks, and customer communication. The representative path is

{\small
    \begin{equation}
    \begin{split}
    \pi_{\mathrm{fa}}={}& [(D,\textsf{corruption}),(C,\textsf{corruption}),\\
    & (T,\textsf{delay}),(E,\textsf{coupling\_abuse})].
    \end{split}
    \end{equation}
}
A corrupted record or prompt payload enters the data layer, influences compute logic, interacts with delayed review timing, and propagates into the external workflow state. Here, immediate catastrophic operational loss may remain moderate. However, weak logging, broken explanation chains, or incomplete reviewer traceability can keep occupancy in degraded but weakly observable states non-trivial. Thus, $K(\pi_{\mathrm{fa}},\tau,u)$ may improve the operational term, while $\Rgov(\pi_{\mathrm{fa}},\tau,u)$ remains positive. The scenario shows why composable trust cannot be reduced to catastrophe avoidance alone. 
Here $\tau=0.75$\,s, $w_{\mathrm{gov}}=0.90$, \texttt{normalize\_gov=true}: $K=1.339$, $\pi_F=0.223$, $\Rres=966{,}412$. Fault-injection simulation ($N=5{,}000$ SSA trajectories, seed~42): mean $954{,}231$, p95 $2{,}348{,}066$.
At this sub-second horizon $\Rgov$ is numerically small. The scenario demonstrates the structural governance mechanism rather than a quantitative $\Rgov$ magnitude (see short-horizon disclosure in Section~\ref{sec:sensitivity}).

Both scenarios are summarized in Table~\ref{tab:cases}.
\begin{table*}[t]
\caption{Illustrative mapping from failure path to quantified residual risk}
\label{tab:cases}
\centering
\begin{tabular}{p{2.2cm}p{3.2cm}p{2.4cm}p{2.6cm}p{1.5cm}p{2.8cm}}
\toprule
Case & Path signature & Terminal consequence & Dominant control interpretation & $K$ effect & Residual-score behavior \\
\midrule
Warehouse robot & $S \rightarrow D \rightarrow C \rightarrow T \rightarrow A$ & Physical collision or unsafe motion & Fast containment, interlock, fallback timing & Large increase & Base \friesak\ term drops sharply when containment outruns escalation \\
Financial-services agent & $D \rightarrow C \rightarrow T \rightarrow E$ & Workflow misuse, assurance loss, auditability degradation & Tool guardrails plus observability instrumentation & Moderate increase & Operational term dominates at $\tau=0.75$\,s; $\Rgov$ structurally non-zero but numerically minor at this horizon; governance penalty grows material at longer evaluation windows \\
\bottomrule
\end{tabular}
\end{table*}

\subsection{Sensitivity and Discussion}
\label{sec:sensitivity}

The framework exposes direct comparative sensitivities. For fixed $\Ri,E,S,A$, the \friesak\ functional is linear in $F$ and inverse in $K$:
\[
\risk \propto F,
\qquad
\risk \propto \frac{1}{K}.
\]
Thus, larger attack frequency increases residual risk, while stronger control effectiveness reduces it.

\begin{figure}[t]
\centering
\includegraphics[width=\columnwidth]{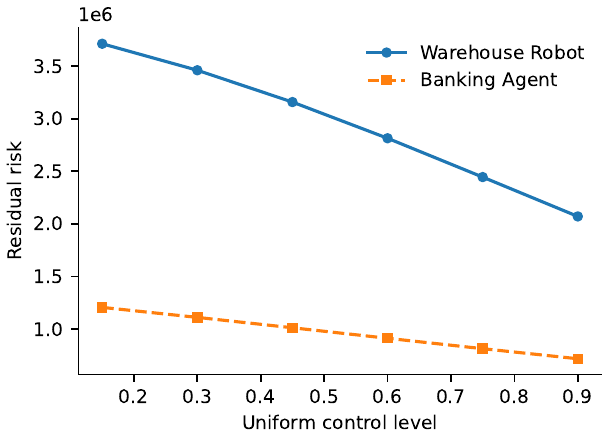}
\caption{Sensitivity sweep: residual risk $\Rres$ versus uniform control level for the
warehouse-robot (response-dominant, $w_{\mathrm{gov}}=0.35$) and banking-agent
(governance-instrumented, $w_{\mathrm{gov}}=0.90$) scenarios.
Both curves use $\tau=0.75$\,s and \texttt{normalize\_gov=true}.
Deterministic path; bundle-default parameter values.}
\label{fig:k_sweep_det}
\end{figure}

Figure~\ref{fig:k_sweep_det} shows the deterministic sensitivity
of $\Rres$
to uniform control level for both cases.
The warehouse-robot residual risk drops sharply with improving controls
because high containment and recovery elasticities
($\beta_c = 2.1$, $\beta_r = 1.8$) translate small control
improvements into large reductions in catastrophic absorption
probability.
The banking-agent curve is flatter because lower hazard rates and a smaller
$S\times A$ product ($0.9\times10^6\times1.25$ vs.\ $1.8\times10^6\times1.65$)
reduce the baseline operational term and limit the K-induced reduction.
The governance-observability weight is set to $w_{\mathrm{gov}}=0.90$ to
reflect governance priority, but at $\tau=0.75$\,s the additive $\Rgov$
term is numerically small relative to the operational term; see the
short-horizon disclosure below.
The warehouse robot carries higher absolute $\Rres$ despite the lower governance weight. Steeper hazard rates and a larger $S\times A$ product dominate. $K$ gain alone cannot close that gap.

\begin{figure}[t]
\centering
\includegraphics[width=\columnwidth]{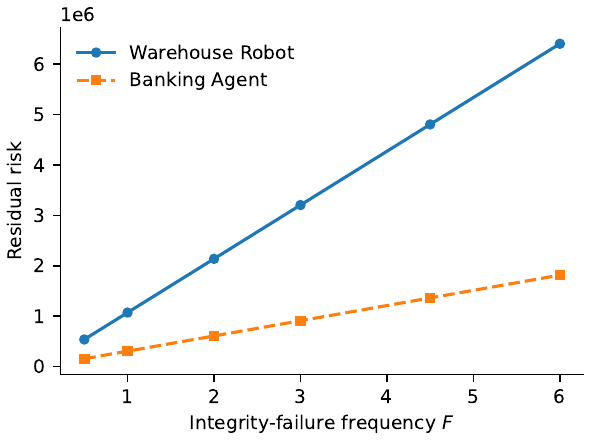}
\caption{Residual risk $\Rres$ as a function of integrity-failure
frequency $F$, all other parameters held at bundle nominals.
Both cases scale linearly with $F$, consistent with the analytical form of \eqref{eq:friesak}.
The warehouse robot shows a higher slope because $S(\pi_{\mathrm{wr}})\times
A(\pi_{\mathrm{wr}})$ is larger. The banking agent shows a shallower slope
because lower amplification ($\mu_A=1.25$ vs.\ $1.65$) and lower severity
keep the numerator product modest even at high event rates.}
\label{fig:freq_sweep}
\end{figure}
Figure~\ref{fig:freq_sweep} sweeps $F$ while holding all other parameters at bundle defaults.
Both curves are linear in $F$ as the risk formula requires. 
The warehouse robot
slope exceeds the banking agent slope because $S \times A$ is larger in the
physical case ($1.8\times10^{6} \times 1.65$ vs.\ $0.9\times10^{6} \times 1.25$).

One practical note merits disclosure. At $\tau=0.75$\,s, the governance-observability penalty $\Rgov$ is numerically negligible relative to the \friesak\ operational term for both scenarios. This is expected: a sub-second horizon provides insufficient dwell time for unobservable compromised states to accumulate a material integral. The banking-agent scenario demonstrates the governance \emph{structural mechanism}. The framework represents that a governance-priority system requires separate observability accounting. But the magnitude of $\Rgov$ becomes material only at longer evaluation horizons. 
Clinical workflow scenarios (e.g., a 12-hour post-surgical shift, $\tau=43{,}200$\,s) exhibit substantial $\Rgov$ contributions. 
Fault-injection runs on the oncology delegation bundles at this horizon confirm that $\Rgov$ becomes the dominant term when governance weight $w_{\mathrm{gov}}=0.85$ and unobservable dwell time accumulates over hours
  rather than milliseconds.

Therefore, these comparative cases are best read  as a demonstration of the operational mechanics of \friesak. The governance dimension is demonstrated at scale via longer-horizon bundles.

The evaluation supports three conclusions. First, the same \cpsaint\ grammar expresses materially different paths without changing the layer set. Second, the same \friesak\ semantics remain intact across both cases. Third, dynamic resistance and governance observability separate two conditions that static scoring often conflates: operational safety improvement and residual assurance degradation. That distinction matters. It is the reason the framework establishes a composable trust and avoids monolithic brittleness.

Existing baselines offer partial views. Structural frameworks (e.g., STPA) identify hazards without quantifying residual risk, while probabilistic estimators (e.g., CVaR) quantify risk but black-box the failure path. Neither supports targeted control optimization across end-to-end workflows. 
Runtime-assurance methods such as AgentSpec enforce per-agent constraints with binary guarantees but do not propagate quantitative residual risk across composed paths.
By coupling \cpsaint\ with \friesak, we provide a concise mechanism to evaluate competing interventions. That is, we mathematically distinguish the risk-reduction value of data-layer detection ($D$) versus actuation-layer containment ($A$). Thus, qualitative hazard paths become measurable risk instances.

\subsection{Uncertainty and One-at-a-Time Weak-Control Comparison}
\label{sec:uncertainty-ablation}

Figure~\ref{fig:k_sweep_unc} extends the deterministic sensitivity view in Figure~\ref{fig:k_sweep_det} with Monte Carlo mean curves and p50--p90 bands. The uncertainty envelope widens the absolute risk range but does not reverse the qualitative ordering between the two scenarios. The warehouse-robot band is broader in absolute terms because uncertainty in severity and amplification multiplies through a larger nominal loss scale. The banking-agent curve remains flatter across the control grid, which is consistent with lower operational hazard rates and a comparatively smaller shift in the operational term as uniform controls improve.

\begin{figure}[t]
\centering
\includegraphics[width=\columnwidth]{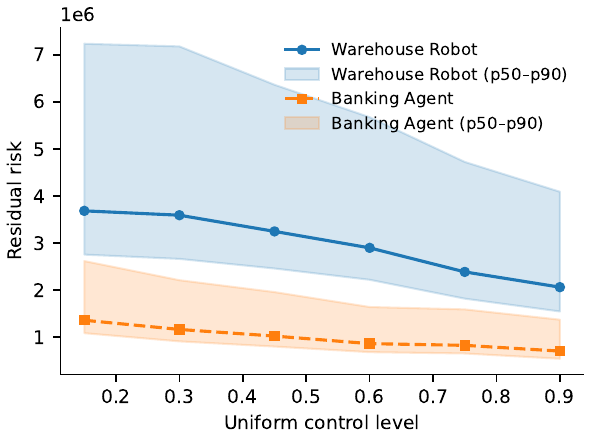}
\caption{Residual risk $\Rres$ versus uniform control level with Monte Carlo mean curves and p50--p90 bands for both scenarios. Uncertainty widens the envelope but preserves the qualitative ordering observed in the deterministic sweep.}
\label{fig:k_sweep_unc}
\end{figure}

Figure~\ref{fig:ablation} reports the absolute reported Residual Risk $\Rres$ under baseline and one-at-a-time weakened-control conditions. The grouped bars should be read as a weak-control comparison, not as a fractional-removal analysis. For the warehouse robot, the largest visible increase occurs under weak response and weak detection. This is consistent with a path in which containment and time-critical intervention dominate operational risk. For the banking agent, weak detect produces the largest visible increase in total $\Rres$, while weak response and weak recovery move the score only modestly.

\begin{figure}[t]
\centering
\includegraphics[width=\columnwidth]{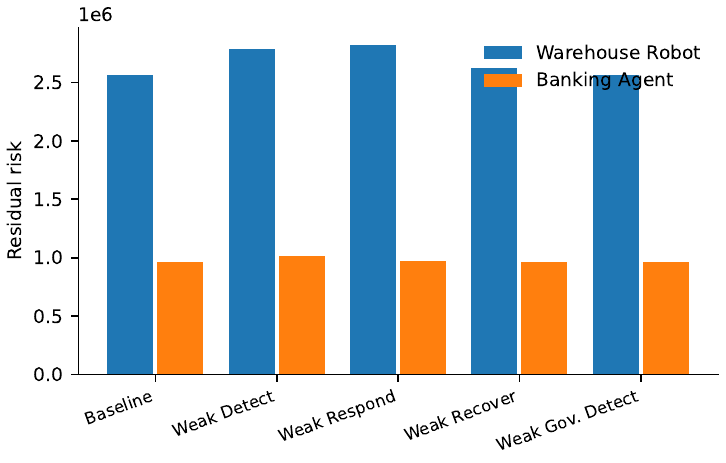}
\caption{One-at-a-time weak-control comparison for the warehouse-robot and banking-agent scenarios. Each bar shows the absolute reported score $\Rres$ under the baseline bundle and under isolated weakened-control conditions.}
\label{fig:ablation}
\end{figure}

The governance result requires a narrower reading. Weak governance detection produces little visible movement in total $\Rres$ at $\tau=0.75$\,s for either case, even though governance observability remains structurally modeled. That behavior is consistent with the short-horizon disclosure above: over a sub-second window, the additive $\Rgov(\pi,\tau,u)$ term remains numerically small relative to the operational \friesak\ term. Thus, the present ablation supports a bounded claim. The current short-horizon bundles are most sensitive to response and operational detection on the warehouse path, and to operational detection on the banking path, while governance-heavy effects become materially visible only under longer-horizon instantiations.

\section{Limitations}
\label{sec:limits}

The controlled CTMC assumes a finite-state abstraction rich enough to capture the materially relevant modes for the hazard class, however calibration remains domain-specific. We demonstrate instantiation, not a universally calibrated estimator. Correlation among numerator terms and overlap among controls can bias naive factorizations if left unjustified. We therefore use this paper to lock semantics and prove the mechanism-to-magnitude mapping, not to claim complete dependence modeling. We also keep the seven-layer decomposition intentionally compact. Some domains may require finer internal structure inside a layer. That refinement can be added without changing the core grammar.

\section{Conclusion}
\label{sec:conclusion}

Agentic AI requires a framework that links the structural mechanism of 
failure to the magnitude of residual risk with fixed semantics that 
transfer across domains. \cpsaint\ provides the structural half: a 
seven-layer grammar with an explicit propagation relation and a 
five-mode failure alphabet. \friesak\ provides the quantitative half: 
locked variable boundaries, a dynamic resistance term grounded in 
absorbing Markov state behavior, and an additive governance-observability 
penalty that extends the framework without redefining the canonical 
risk equation. The composition theorem makes the linkage formal, and 
the two contrasting case studies confirm it in practice. The result is 
a compact, cross-domain kernel for mechanism-to-magnitude reasoning 
with explicit assumptions and composable trust semantics.

\section{Future Work}
\label{sec:future_work}

Several open problems remain. First, empirical calibration of 
\friesak\ parameters in operational deployments would replace 
expert-assigned values with measured priors. Second, the current model treats the numerator terms as independent, and in so doing, modeling correlation between $F$, $R_i$, $E$, $S$, and $A$ is a natural extension. 
Third, multi-agent settings introduce overlapping control surfaces that the single-path composition theorem does not yet cover. 
Finally, automated path extraction from system architecture artifacts would reduce the manual effort currently required to instantiate a 
\cpsaint\ decomposition.

\bibliographystyle{IEEEtran}
\bibliography{references}

@misc{Gartner2025AgenticCancel,
    author       = {{Gartner}},
    title        = {Gartner Predicts Over 40\% of Agentic {AI} Projects Will Be Canceled by End of 2027},
    howpublished = {\url{https://www.gartner.com/en/newsroom/press-releases/2025-06-25-gartner-predicts-over-40-percent-
  of-agentic-ai-projects-will-be-canceled-by-end-of-2027}},
    month        = jun,
    year         = {2025},
    note         = {Accessed: 2026-04-09}
  }

@misc{dehghantanha2026sokatksurfagentic,
      title={SoK: The Attack Surface of Agentic AI -- Tools, and Autonomy}, 
      author={Ali Dehghantanha and Sajad Homayoun},
      year={2026},
      eprint={2603.22928},
      archivePrefix={arXiv},
      primaryClass={cs.CR},
      url={https://arxiv.org/abs/2603.22928}, 
}

@article{Leveson2004STAMP,
  author  = {Leveson, Nancy G.},
  title   = {A New Accident Model for Engineering Safer Systems},
  journal = {Safety Science},
  volume  = {42},
  number  = {4},
  pages   = {237--270},
  year    = {2004},
  doi     = {10.1016/S0925-7535(03)00047-X}
}

@article{Friedberg2017STPASafeSec,
  author  = {Friedberg, Ivo and McLaughlin, Kieran and Smith, Paul and Laverty, David and Sezer, Sakir},
  title   = {{STPA-SafeSec}: Safety and Security Analysis for Cyber-Physical Systems},
  journal = {Journal of Information Security and Applications},
  volume  = {34},
  pages   = {183--196},
  year    = {2017},
  doi     = {10.1016/j.jisa.2016.05.008}
}

@article{Kordy2014ADTrees,
  author  = {Kordy, Barbara and Mauw, Sjouke and Radomirovi\'{c}, Sa\v{s}a and Schweitzer, Patrick},
  title   = {Attack--Defense Trees},
  journal = {Journal of Logic and Computation},
  volume  = {24},
  number  = {1},
  pages   = {55--87},
  year    = {2014},
  doi     = {10.1093/logcom/exs029}
}

@inproceedings{Khattab2024DSPy,
  author={Khattab, Omar and Singhvi, Arnav and Maheshwari, Paridhi and Zhang, Zhiyuan and Santhanam, Keshav and Vardhamanan, Sri and Haq, Saiful and Sharma, Ashutosh and Joshi, Thomas T and Moazam, Hanna and others},
  title     = {{DSPy}: Compiling Declarative Language Model Calls into Self-Improving Pipelines},
  booktitle = {Proc. Int. Conf. Learning Representations (ICLR)},
  year      = {2023},
  note      = {arXiv:2310.03714}
}

@inproceedings{Wu2024AutoGen,
  author    = {Wu, Qingyun and Bansal, Gagan and Zhang, Jieyu and Wu, Yiran and Li, Beibin and Zhu, Erkang and Jiang, Li and Zhang, Xiaoyun and Zhang, Shaokun and Liu, Jiale and others},
  title     = {{AutoGen}: Enabling Next-Gen {LLM} Applications via Multi-Agent Conversation},
  booktitle = {Proc. Conf. Language Modeling (COLM)},
  year      = {2024},
  note      = {arXiv:2308.08155}
}

@inproceedings{Hong2024MetaGPT,
  author    = {Hong, Sirui and Zhuge, Mingchen and Chen, Jonathan and Zheng, Xiawu and Cheng, Yuheng and Zhang, Ceyao and Wang, Jinlin and Wang, Zili and Yau, Steven Ka Shing and Lin, Zijuan and others},
  title     = {{MetaGPT}: Meta Programming for a Multi-Agent Collaborative Framework},
  booktitle = {Proc. Int. Conf. Learning Representations (ICLR)},
  year      = {2024},
  note      = {arXiv:2308.00352}
}

@inproceedings{Cemri2025MAST,
  author    = {Cemri, Mert and Pan, Melissa Z. and Yang, Shuyi and Agrawal, Lakshya A. and Chopra, Bhavya and Albarghouthi, Aws and Jha, Somesh and Lakkaraju, Himabindu},
  title     = {Why Do Multi-Agent {LLM} Systems Fail?},
  booktitle = {Advances in Neural Information Processing Systems (NeurIPS), Datasets and Benchmarks Track},
  year      = {2025},
  note      = {Spotlight. arXiv:2503.13657}
}

@article{RiosInsua2021ARAcyber,
  author  = {Rios Insua, David and Couce-Vieira, Aitor and Rubio, Jose A. and Pieters, Wolter and Labunets, Katsiaryna and Rasines, Daniel G.},
  title   = {An Adversarial Risk Analysis Framework for Cybersecurity},
  journal = {Risk Analysis},
  volume  = {41},
  number  = {1},
  pages   = {16--36},
  year    = {2021},
  doi     = {10.1111/risa.13331}
}

@article{Artzner1999CoherentRisk,
  author  = {Artzner, Philippe and Delbaen, Freddy and Eber, Jean-Marc and Heath, David},
  title   = {Coherent Measures of Risk},
  journal = {Mathematical Finance},
  volume  = {9},
  number  = {3},
  pages   = {203--228},
  year    = {1999},
  doi     = {10.1111/1467-9965.00068}
}

@article{RockafellarUryasev2000CVaR,
  author  = {Rockafellar, R. Tyrrell and Uryasev, Stanislav},
  title   = {Optimization of Conditional Value-at-Risk},
  journal = {Journal of Risk},
  volume  = {2},
  number  = {3},
  pages   = {21--41},
  year    = {2000},
  URL     = {https://doi.org/10.21314/JOR.2000.038}
}

@article{AuBeck2001SubsetSimulation,
  author  = {Au, Siu-Kui and Beck, James L.},
  title   = {Estimation of Small Failure Probabilities in High Dimensions by Subset Simulation},
  journal = {Probabilistic Engineering Mechanics},
  volume  = {16},
  number  = {4},
  pages   = {263--277},
  year    = {2001},
  doi     = {10.1016/S0266-8920(01)00019-4}
}

@article{Buldyrev2010Cascades,
  author  = {Buldyrev, Sergey V. and Parshani, Roni and Paul, Gerald and Stanley, H. Eugene and Havlin, Shlomo},
  title   = {Catastrophic Cascade of Failures in Interdependent Networks},
  journal = {Nature},
  volume  = {464},
  number  = {7291},
  pages   = {1025--1028},
  year    = {2010},
  doi     = {10.1038/nature08932}
}

@article{Karim2025IoRT,
  author  = {Karim, Hassan and Gupta, Deepti and Sitharaman, Sai},
  title   = {Securing {LLM} Workloads With {NIST AI RMF} in the Internet of Robotic Things},
  journal = {IEEE Access},
  volume  = {13},
  pages   = {69631--69649},
  year    = {2025},
  doi     = {10.1109/ACCESS.2025.3561235}
}

@misc{Wang2026AgentSpec,
      title={AgentSpec: Customizable Runtime Enforcement for Safe and Reliable LLM Agents}, 
      author={Haoyu Wang and Christopher M. Poskitt and Jun Sun},
      year={2025},
      eprint={2503.18666},
      archivePrefix={arXiv},
      primaryClass={cs.AI},
      url={https://arxiv.org/abs/2503.18666}, 
}

@misc{Wang2025Pro2Guard,
      title={Pro2Guard: Proactive Runtime Enforcement of {LLM} Agent Safety via Probabilistic Model Checking}, 
      author={Haoyu Wang and Christopher M. Poskitt and Jun Sun and Jiali Wei},
      year={2026},
      eprint={2508.00500},
      archivePrefix={arXiv},
      primaryClass={cs.AI},
      url={https://arxiv.org/abs/2508.00500}, 
}

@misc{Debenedetti2025CaMeL,
      title    = {Defeating Prompt Injections by Design}, 
      author    = {Edoardo Debenedetti and Ilia Shumailov and Tianqi Fan and Jamie Hayes and Nicholas Carlini and Daniel Fabian and Christoph Kern and Chongyang Shi and Andreas Terzis and Florian Tramèr},
      year    = {2025},
      eprint    = {2503.18813},
      archivePrefix    = {arXiv},
      primaryClass    = {cs.CR},
      url    = {https://arxiv.org/abs/2503.18813}, 
}

@misc{Wu2025IsolateGPT,
      title={IsolateGPT: An Execution Isolation Architecture for LLM-Based Agentic Systems}, 
      author={Yuhao Wu and Franziska Roesner and Tadayoshi Kohno and Ning Zhang and Umar Iqbal},
      year={2025},
      eprint={2403.04960},
      archivePrefix={arXiv},
      primaryClass={cs.CR},
      url={https://arxiv.org/abs/2403.04960}, 
}

@misc{Shi2025Progent,
      title={Progent: Programmable Privilege Control for LLM Agents}, 
      author={Tianneng Shi and Jingxuan He and Zhun Wang and Hongwei Li and Linyu Wu and Wenbo Guo and Dawn Song},
      year={2025},
      eprint={2504.11703},
      archivePrefix={arXiv},
      primaryClass={cs.CR},
      url={https://arxiv.org/abs/2504.11703}, 
}

@techreport{BloomfieldRushby2024Assurance,
  author      = {Bloomfield, Robin and Rushby, John},
  title       = {Assurance of {AI} Systems from a Dependability Perspective},
  institution = {SRI International},
  number      = {SRI-CSL-2024-02},
  year        = {2024},
  note        = {arXiv:2407.13948. Companion at ASSURE 2024 (ISSRE 2024)}
}

@inproceedings{Huang2025Resilience,
  title = {On the Resilience of {LLM}-Based Multi-Agent Collaboration with Faulty Agents},
  author = {Huang, Jen-Tse and Zhou, Jiaxu and Jin, Tailin and Zhou, Xuhui and Chen, Zixi and Wang, Wenxuan and Yuan, Youliang and Lyu, Michael and Sap, Maarten},
  booktitle = {Proc. 42nd Int. Conf. Machine Learning (ICML)},
  series    = {PMLR},
  volume    = {267},
  pages     = {26202--26226},
  year      = {2025},
  editor = 	 {Singh, Aarti and Fazel, Maryam and Hsu, Daniel and Lacoste-Julien, Simon and Berkenkamp, Felix and Maharaj, Tegan and Wagstaff, Kiri and Zhu, Jerry},
  month = 	 {13--19 Jul},
  publisher =    {PMLR},
}

@unpublished{karim2026forgebench,
  author = {Karim, Hassan and Gupta, Deepti and Nair, Akarsh K. and Boyd, Daonte and Elluri, Lavanya},
  title = {FORGE-Bench: A Threat-Labeled Dataset and Digital Twin Framework for Security Evaluation of {LLM}-Driven Warehouse Robots},
  year = {2026},
  note = {Manuscript under review},
  url = {https://orcid.org/0000-0002-5441-049X}
}


\end{document}